\documentclass[10pt,conference]{IEEEtran}

\usepackage{pgf,tikz,pgfplots}
\usepackage{tikz-cd}
\pgfplotsset{compat=1.14}
\usetikzlibrary[patterns,arrows,shapes]
\usepackage{multirow,tabu}
\usepackage{color, colortbl}

\usepackage{subfig}

\usepackage{hyperref}
\hypersetup{
	colorlinks,
	linkcolor={blue!100!black},
	citecolor={blue!100!black},
	urlcolor={blue!80!black}
}

\usepackage{amsmath,amssymb,amsfonts,amsthm,mathtools}

\usepackage[numbers,sort]{natbib}

\newtheorem{theorem}{Theorem}

\newtheorem{example}{Example}
\newtheorem{lemma}{Lemma}

\newcommand{\bE}{\mathbb{E}}
\newcommand{\bF}{\mathbb{F}}

\newcommand{\bR}{\mathbb{R}}

\newcommand{\cB}{\mathcal{B}}

\newcommand{\cF}{\mathcal{F}}

\newcommand{\cH}{\mathcal{H}}

\newcommand{\cS}{\mathcal{S}}
\newcommand{\cT}{\mathcal{T}}
\newcommand{\cU}{\mathcal{U}}

\newcommand{\bolda}{\mathbf{a}}
\newcommand{\boldb}{\mathbf{b}}

\newcommand{\boldf}{\mathbf{f}}

\newcommand{\boldv}{\mathbf{v}}
\newcommand{\boldw}{\mathbf{w}}
\newcommand{\boldx}{\mathbf{x}}
\newcommand{\boldy}{\mathbf{y}}
\newcommand{\boldz}{\mathbf{z}}

\newcommand{\boldtau}{\boldsymbol{\tau}}

\newcommand{\onenorm}[1]{\lVert #1 \rVert_1}
\newcommand{\twonorm}[1]{\lVert #1 \rVert_2}
\newcommand{\infnorm}[1]{\lVert #1 \rVert_\infty}

\DeclareMathOperator{\sign}{sign}

\DeclareMathOperator{\logistic}{logistic}
\DeclareMathOperator{\ReLU}{ReLU}
\DeclareMathOperator{\Id}{Id}

\DeclareSymbolFont{bbold}{U}{bbold}{m}{n}
\DeclareSymbolFontAlphabet{\mathbbold}{bbold}
\newcommand{\1}{\mathbbold{1}}

\author{\textbf{Netanel Raviv}$^\star$, \IEEEauthorblockN{\textbf{Siddharth Jain}$^\dagger$, \textbf{Pulakesh Upadhyaya}$^\ddagger$, \textbf{Jehoshua Bruck}$^\dagger$, and \textbf{Anxiao (Andrew) Jiang}$^\ddagger$}
	\IEEEauthorblockA{
		$^\star$Department of Computer Science and Engineering, Washington University in St. Louis, St. Louis 63130, MO, USA\\
		$^\dagger$Department of Electrical Engineering, California Institute of Technology, Pasadena 91125, CA, USA\\
		$^\ddagger$Department of Computer Science and Engineering, Texas A\&M University, College Station 77843, TX, USA \\}}
		
\begin{document}

\title{CodNN -- Robust Neural Networks\\From Coded Classification}

\maketitle

\IEEEpeerreviewmaketitle

\begin{abstract}
Deep Neural Networks (DNNs) are a revolutionary force in the ongoing information revolution, and yet their intrinsic properties remain a mystery. In particular, it is widely known that DNNs are highly sensitive to noise, whether adversarial or random. This poses a fundamental challenge for hardware implementations of DNNs, and for their deployment in critical applications such as autonomous driving.

In this paper we construct robust DNNs via error correcting codes. By our approach, either the data or internal layers of the DNN are coded with error correcting codes, and successful computation under noise is guaranteed. Since DNNs can be seen as a layered concatenation of classification tasks, our research begins with the core task of classifying noisy coded inputs, and progresses towards robust DNNs.

We focus on binary data and linear codes. Our main result is that the prevalent \textit{parity code} can guarantee robustness for a large family of DNNs, which includes the recently popularized \textit{binarized neural networks}. Further, we show that the coded classification problem has a deep connection to Fourier analysis of Boolean functions.

In contrast to existing solutions in the literature, our results do not rely on altering the training process of the DNN, and provide mathematically rigorous guarantees rather than experimental evidence.
\end{abstract}

\renewcommand{\thefootnote}{\arabic{footnote}}

\section{Introduction}\label{section:introduction}
Deep Neural Networks (DNNs) have become a dominating force in Artificial Intelligence (AI), bringing revolutions in science and technology. A massive amount of academic and industrial research is being devoted to implementing DNNs in hardware \cite{Guo}. Hardware-implemented DNNs are appearing in phones, sensors, healthcare devices, and more, which will revolutionize every sector of society~\cite{Mohammadi}, and make AI systems increasingly energy-efficient and ubiquitous.

In parallel, DNNs are known to be highly susceptible to adversarial interventions. In a recent line of works which followed~\cite{Szegedy}, it was shown that by adding a small (and often indistinguishable to humans) amount of noise to the \textit{inputs} of a DNN, one can cause it to reach nonsensical conclusions. More recently, it was shown~\cite{Shamir} that in some DNN architectures one can attain similar effects by changing as little as one or two entries of the input. This reveals an orthogonal concern of a similar nature from the adversarial machine learning perspective: performance degradation due to malicious attacks. 

\begin{figure*}[t]
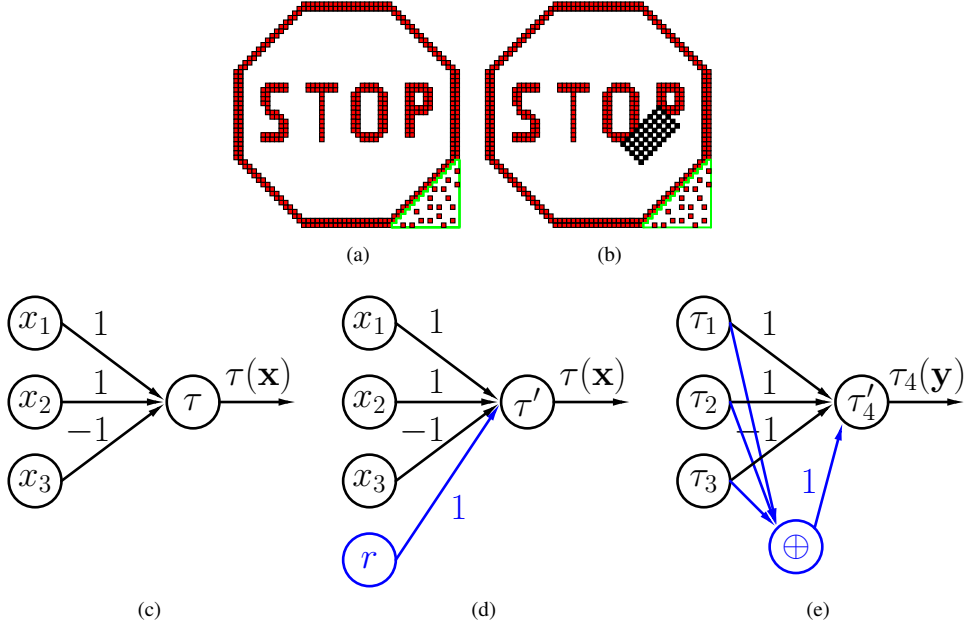

	\centering
	\subfloat[]{%
		\input{TiKz_STOP_Redundancy.tex}
		\label{figure:STOP_Redundancy}%
	}
	\subfloat[]{%
		\input{TiKz_STOP_StickerAndRedundancy.tex}
		\label{figure:STOP_StickerAndRedundancy}%
	}\\
	\subfloat[]{%
		\begin{tikzpicture}[line cap=round,line join=round,x=1cm,y=1cm, scale=0.35,every node/.style={scale=0.67}]
\clip(0,-0.1) rectangle (12,11.1);
\node at (2,10) {\huge{$x_1$}};
\node at (2,7) {\huge{$x_2$}};
\node at (2,4) {\huge{$x_3$}};
\node at (10.5,8) {\huge{$\tau(\boldx)$}};
\draw [line width=1pt] (2,4) circle (1cm);
\draw [line width=1pt] (2,7) circle (1cm);
\draw [line width=1pt] (2,10) circle (1cm);

\node at (4.5,10) {\huge{$1$}};
\node at (4.5,7.75) {\huge{$1$}};
\node at (4,5.9) {\huge{$-1$}};
\node at (8,7) {\huge{$\tau$}};

\draw [line width=1pt] (8,7) circle (1cm);

\draw [-{Latex[width=1mm]},line width=1pt] (3,10) -- (7,7);
\draw [-{Latex[width=1mm]},line width=1pt] (3,7) -- (7.,7.);
\draw [-{Latex[width=1mm]},line width=1pt] (3,4.) -- (7.,7.);

\draw [-{Latex[width=1mm]},line width=1pt] (9,7) -- (12,7);

\end{tikzpicture}
		\label{figure:UncodedToyNeuron}%
	} 
	\subfloat[]{%
		\begin{tikzpicture}[line cap=round,line join=round,x=1cm,y=1cm, scale=0.35,every node/.style={scale=0.67}]
\clip(0,-0.1) rectangle (12,11.1);
\node at (2,10) {\huge{$x_1$}};
\node at (2,7) {\huge{$x_2$}};
\node at (2,4) {\huge{$x_3$}};
\node at (2,1) {\huge{\textcolor{blue}{$r$}}};
\node at (10.5,8) {\huge{$\tau(\boldx)$}};
\draw [line width=1pt] (2,4) circle (1cm);
\draw [line width=1pt] (2,7) circle (1cm);
\draw [line width=1pt] (2,10) circle (1cm);
\draw [line width=1pt,color=blue] (2,1) circle (1cm);
\node at (8,7) {\huge{$\tau'$}};

\node at (4,10) {\huge{$\phantom{-}1$}};
\node at (4,7.75) {\huge{$\phantom{-}1$}};
\node at (4,5.9) {\huge{$-1$}};
\node at (5.4,3) {\huge{\textcolor{blue}{$1$}}};

\draw [line width=1pt] (8,7) circle (1cm);

\draw [-{Latex[width=1mm]},line width=1pt] (3,10) -- (7,7);
\draw [-{Latex[width=1mm]},line width=1pt] (3,7) -- (7.,7.);
\draw [-{Latex[width=1mm]},line width=1pt] (3,4.) -- (7.,7.);
\draw [-{Latex[width=1mm]},line width=1pt,color=blue] (3,1.) -- (7.,7.);

\draw [-{Latex[width=1mm]},line width=1pt] (9,7) -- (12,7);

\end{tikzpicture}
		\label{figure:ToyNeuronInput}%
	}
	\subfloat[]{%
		\begin{tikzpicture}[line cap=round,line join=round,x=1cm,y=1cm, scale=0.35,every node/.style={scale=0.67}]
\clip(0,-0.1) rectangle (12,11.1);
\node at (2,10) {\huge{$\tau_1$}};
\node at (2,7) {\huge{$\tau_2$}};
\node at (2,4) {\huge{$\tau_3$}};
\node at (5.5,1.5) {\textcolor{blue}{\huge{$\oplus$}}};
\node at (10.5,8) {\huge{$\tau_4(\boldy)$}};
\node at (8,7) {\huge{$\tau_4'$}};
\draw [line width=1pt] (2,4) circle (1cm);
\draw [line width=1pt] (2,7) circle (1cm);
\draw [line width=1pt] (2,10) circle (1cm);
\draw [line width=1pt, color=blue] (5.5,1.5) circle (1cm);

\node at (4.5,10) {\huge{$1$}};
\node at (4.5,7.75) {\huge{$1$}};
\node at (4,5.9) {\huge{$-1$}};

\draw [line width=1pt] (8,7) circle (1cm);

\draw [-{Latex[width=1mm]},line width=1pt] (3,10) -- (7,7);
\draw [-{Latex[width=1mm]},line width=1pt, color=blue] (3,10) -- (4.79,2.207);

\draw [-{Latex[width=1mm]},line width=1pt] (3,7) -- (7.,7.);
\draw [-{Latex[width=1mm]},line width=1pt, color=blue] (3,7) -- (4.79,2.207);

\draw [-{Latex[width=1mm]},line width=1pt] (3,4.) -- (7.,7.);
\draw [-{Latex[width=1mm]},line width=1pt, color=blue] (3,4) -- (4.79,2.207);

\draw [-{Latex[width=1mm]},line width=1pt, color=blue] (6.207,2.207) -- (7.292,6.292);

\draw [-{Latex[width=1mm]},line width=1pt] (9,7) -- (12,7);

\node at (6,4) {\huge{\textcolor{blue}{$1$}}};

\end{tikzpicture}
		\label{figure:InnerRedundancy}%
	}
	\caption{(a) a stop sign with added redundancy; (b) the added redundancy guarantees correct classification under noise, e.g., a sticker; (c) an uncoded neuron; (d) a coded neuron with one extra bit of redundancy, guaranteed to compute~$\tau(\boldx)$ correctly, even if \textit{any} synapse is erased (see Example~\ref{example:codedNeuron}); (e) added redundancy \textit{inside} the DNN--the coded neuron~$\tau_4'$ computes~$\tau_4(\boldy)$ even if any of its incoming synapses is erased, where~$\boldy$ is the output of the neurons $\tau_1,\tau_2$, and~$\tau_3$ in the previous layer. %(c) an uncoded neuron which computes~$\tau(\boldx)=\sign(x_1+x_2-x_3)$; (d) a coded version~$\tau(\boldx)$, which computes~$\sign(x_1+x_2-x_3+r)$, where~$r=x_1x_2x_3$ is a redundant bit. This coded neuron is guaranteed to compute~$\tau(\boldx)$ correctly, even if \textit{any} synapse is erased.
	}
	\label{figure:illustration}
\end{figure*}

There exists a rich body of research which studies how to make DNNs robust to noise. This includes noise that is injected into the neurons/synapses, or into the inputs. Even though computation under noise has been studied since the 1950's~\cite{vonNeumann}, solutions have been almost exclusively heuristic. 

%The input corruption front is categorized as \textit{decision time attacks} in the adversarial machine learning literature~\cite{Eugene} (unlike \textit{poisoning attacks}, or attacks on the learning algorithm). 
To combat adversarial attacks to the inputs, much focus was given on adjusting the \textit{training} process to produce more robust DNNs, e.g., by adjusting the regularization expression~\cite{NNattack2}, or the loss function~\cite{Madry}. These approaches usually involve intractable optimization problems, and succeed insofar as the underlying optimization succeeds.

Combating noise in neurons/synapses has also enjoyed a recent surge of interest~\cite{Torres}, which builds upon the previous wave of interest in DNNs in the early 1990's~\cite{Koren}. Most of this line of research focuses on replication methods (called augmentation), retraining, and providing statistical frameworks for testing fault tolerance of DNNs (e.g., training a DNN to remember a coded version of all possible outputs~\cite{Ito}). It is also worth mentioning that to a certain degree, DNNs tend to present some natural fault-tolerance without any intervention. This phenomenon is conjectured to be connected to \textit{over-provisioning}~\cite{El-Mhamdi}, i.e., the fact that in most cases one uses more neurons than necessary, but rigorous guarantees remain elusive.

%\red{A paragraph about our approach. Reference two figures: coded neuron, coded input.} 
In this paper we propose a fundamentally new approach for handling noise in DNNs. In this approach, the inputs to neurons are coded by using an error correcting code, and the neurons are modified accordingly, so that correct output is guaranteed as long as the noise level is below a certain threshold. Clearly, the encoding function must be simpler than conducting the computation itself, and should apply universally to a large family of neurons; we also aim for an efficient end-to-end design that does not require decoding. %\blue{Further, it is clear that some noise in coded input can be cleared by decoding, but we aim for an end-to-end approach that does not require decoding.}

Our approach is depicted in Fig.~\ref{figure:illustration} for the famous case-study of stop-sign classification~\cite{STOPsign}. %, which is of critical importance in the autonomous driving industry~\cite{STOPsign}. 
In this case-study, an autonomous vehicle uses a DNN to classify a stop sign as such, exposing the passenger to the perils of misclassification due to noise (e.g., a sticker). In our approach we envision addition of redundancy \textit{to the actual physical object} (Fig.~\ref{figure:STOP_Redundancy}), which aids the DNN in classification under noise (Fig.~\ref{figure:STOP_StickerAndRedundancy}). To facilitate this vision, the neurons inside the network must be revised accordingly, e.g., by adding weighted synapses (Fig.~\ref{figure:UncodedToyNeuron} and Fig.~\ref{figure:ToyNeuronInput}). Alternatively, the inputs to some neurons might come from other neurons inside the network, rather than from the physical world; this case better encapsulates hardware failures, and redundancy is computed by additional components inside the DNN (Fig.~\ref{figure:InnerRedundancy}).

Preliminaries are discussed in Section~\ref{section:preliminaries}, where a tight connection to the~$\ell_1$-metric is revealed. Simple but inefficient solutions, that rely on replication or Fourier analysis, are discussed in Section~\ref{section:simple}. Finally, the main result of this paper is given in Section~\ref{section:parity}, where it is shown that the well-known parity code can guarantee successful classification under noise, and applications to a large family of DNNs are discussed.

 \section{Preliminaries}\label{section:preliminaries}
A DNN is a layered and directed acyclic graph, in which the first layer is called \textit{the input layer}. Each edge (or \textit{synapse}) corresponds to a real number called \textit{weight}, and nodes in intermediate layers (also called \textit{hidden layers}) are called \textit{neurons}. Each neuron acts as computation unit, that applies some \textit{activation function} on its inputs, weighted by the respective synapses. The result of the computation in the last layer are the outputs of the DNN.

Traditionally, the activation function is~$\sign(\boldx\boldw^\top-\theta)$, where~$\boldx$ is a vector of inputs, $\boldw$ is a vector of weights, $\theta$ is a constant called~\textit{bias}, and
\begin{align*}
	\sign(x)=
	\begin{cases}
	\phantom{-}1 & \mbox{ if }x\ge 0\\
	-1 & \mbox{ if }x< 0
	\end{cases}.
\end{align*}
However, contemporary DNNs often employ continuous approximations of~$\sign(\cdot)$, known as \textit{sigmoid} functions (such as~$\logistic(x)=\frac{1}{1+\exp(-x)}$ or $\tanh(x)$) or piecewise linear alternatives (such as $\ReLU(x)=\max\{0,x\}$) in order to enable analytic learning algorithms (such as \textit{backpropagation}). In our work, in order to establish rigorous and discrete guarantees, we focus on~$\sign(\cdot)$. Further, we focus on \textit{binary}~$\pm1$ inputs, which correspond to the binary field~$\bF_2$ by identifying~$0$ as~$1$ and~$1$ as~$-1$, and exclusive or as product.

At the computational level, faults in DNN hardware appear as bit-errors, bit-erasures or analog noise, which can be permanent or transient. In this work we focus on bit-errors and bit-erasures, that are formally defined shortly. We denote scalars by lowercase letters~$a,b,\ldots$, vectors by lowercase boldface letters~$\bolda,\boldb,\ldots$, and use the same letter to refer to a vector and its entries (e.g., $\bolda=(a_1,a_2,\ldots)$). We use~$d_i(\cdot,\cdot)$, $\lVert\cdot\rVert_i$, and~$\cB_i(\boldz,r)$ to denote the~$\ell_i$-distance, $\ell_i$-norm, and~$\ell_i$-ball cantered at~$\boldz$ with radius~$r$, respectively, for $i\in\{0,1,2,\ldots,\infty\}$. We use~$d_H(\cdot,\cdot)$ to denote Hamming distance, and let~$\1_m$ be a vector of~$m$ ones.
%\vspace{-1cm}
\subsection{Framework and problem definition}
For a given neuron~$\tau:\bF_2^n\to \bF_2$, where $\tau(\boldx)=\sign(\boldx\cdot\boldw^\top-\theta)$ for some~$\boldw\in\bR^n$ and~$\theta\in\bR$, a triple~$(E,\boldv,\mu)$ is called a \textit{solution}, where~$E:\bF_2^n\to\bF_2^m$, $\boldv\in\bR^m$, and~$\mu\in\bR$. The respective \textit{coded neuron} is~$\tau'(E(\boldx))=\sign(E(\boldx)\boldv^\top-\mu)$. For nonnegative integers~$t$ and~$s$, the coded neuron~$\tau'$ is robust against~$t$ erasures and~$s$ errors ($(t,s)$-robust, for short) if for every disjoint~$t$-subset $\cT\subseteq[m]$, and~$s$-subset~$\cS\subseteq[m]$, we have that
\begin{align}\label{equation:problem}
&\sign(\boldx\cdot\boldw^\top-\theta)=\nonumber\\&\sign\left(\sum_{j\in[m]\setminus(\cT\cup\cS)}E(\boldx)_jv_j-\sum_{j\in\cS}E(\boldx)_jv_j-\mu\right)
\end{align}
for every~$\boldx\in\bF_2^n$. 

Namely, when computing over data encoded by~$E$, correct output for \textit{every}~$\boldx\in\bF_2^n$ is guaranteed, even if at most~$t$ of the inputs to~$\tau'$ are omitted (erasures), and at most~$s$ are negated (errors). Further, for a nonnegative integer~$r$ we say that~$\tau'$ is~$r$-robust if it is~$(t,s)$-robust for every nonnegative~$t$ and~$s$ such that~$t+2s\le r$.

For~$\boldv\in\bR^m$ and~$\mu\in\bR$, let~$\cH(\boldv,\mu)=\{\boldy\in\bR^m |\boldy\boldv^\top=\mu \}$. For a given solution~$(E,\boldv,\mu)$, we say that the \textit{minimum distance} of the respective coded neuron is
\begin{align*}
	d&=d(E,\boldv,\mu)=d_1(E(\bF_2^n),\cH(\boldv,\mu))\\
	&=\min_{\boldx\in\bF_2^n}d_1(E(\boldx),\cH(\boldv,\mu)).
\end{align*}
The choice of the~$\ell_1$-metric will be made clear in the sequel. The figure of merit by which we measure a given solution is its \textit{relative distance}~$d/m$.

\begin{example}\label{example:useless}
	For a given neuron~$\tau$, and an integer~$m$, let 
	\begin{align*}
		E(\boldx)=
		\begin{cases}
		\phantom{-}\1_m & \mbox{if }\tau(\boldx)=1\\
		-\1_m & \mbox{if }\tau(\boldx)=-1\\
		\end{cases}.
	\end{align*} 
	It is readily verified that the solution~$(E,\1_m,0)$ is~$(m-1)$-robust.
\end{example}

Since layers in DNNs normally contain multiple neurons, the solution in Example~\ref{example:useless} is useless for constructing robust DNNs. Instead, one would like to have \textit{joint coding}~$E$ for a large family of neurons.

\textbf{Problem Definition:} For a given set of neurons~$\{\tau_i(\boldx)=\sign(\boldx\boldw_i^\top-\theta_i)\}_{i=1}^\ell$ find a joint coding function~$E$ and~$\{ \boldv_i,\mu_i \}_{i=1}^\ell$ which maximize~$d_{\min}/m$, where~$d_{\min}=\min_{i\in[\ell]} d(E,\boldv_i,\mu_i)$.

Furthermore, we restrict our attention to functions~$E$ which encode binary linear codes, due to their prevalence in hardware and ease of analysis. Since we use the~$\{\pm1\}$-representation of~$\bF_2$, every entry of~$E(\boldx)$ is a multilinear monomial in the entries of~$\boldx$.

\definecolor{LightCyan}{rgb}{0.88,1,1}
\begin{example}\label{example:codedNeuron}
	Fig.~\ref{figure:UncodedToyNeuron} depicts the uncoded neuron~$\tau(\boldx)=\sign(x_1+x_2-x_3)$, and Fig.~\ref{figure:ToyNeuronInput} depicts its coded version~$\tau'(\boldx)=\sign(x_1+x_2-x_3+r)$, where~$r=x_1x_2x_3$. Table~\ref{table:noisy} shows two examples of robustness to any~$1$-erasure.
	\begin{table}[h]
		\small
		\centering
		\begin{tabular}{|c|c|c|}
			\rowcolor{LightCyan}
			\hline
			$(x_1,x_2,x_3,r)$ & Erasure & $\tau'(\text{noisy }E(\boldx))=\tau(\boldx)$ \\ \hline\hline
			&   $x_1$ & $\sign(0-1-1-1)=-1$ \\ \cline{2-3} 
			&   $x_2$ & $\sign(1-0-1-1)=-1$ \\ \cline{2-3} 
			&   $x_3$ & $\sign(1-1-0-1)=-1$ \\ \cline{2-3} 
			\multirow{-4}{*}{$(1,-1,1,-1)$} & $r$ & $\sign(1-1-1-0)=-1$ \\ \hline
			&   $x_1$ & $\sign(-0+1+1+1)=1$ \\ \cline{2-3} 
			&   $x_2$ & $\sign(-1+0+1+1)=1$ \\ \cline{2-3} 
			&   $x_3$ & $\sign(-1+1+0+1)=1$ \\ \cline{2-3} 
			\multirow{-4}{*}{$(-1,1,-1,1)$} & $r$ & $\sign(-1+1+1+0)=1$ \\ \hline
		\end{tabular}
		\caption{Two examples of correct computation of~$\tau(\cdot)$ (Figure~\ref{figure:UncodedToyNeuron}) by~$\tau'(E(\cdot))$ (Figure~\ref{figure:ToyNeuronInput}). This holds for the remaining six inputs as well.}
		\label{table:noisy}
	\end{table}
\end{example}

\subsection{Robustness and the \texorpdfstring{$\ell_1$}{l1}-metric}
In this section we justify the above definitions, and in particular, the use of the~$\ell_1$-metric to obtain robustness. First, notice that errors and erasures while evaluating~$\tau'$ can be seen as changes in~$E(\boldx)$. For example, let~$\boldv=(v_1,v_2,v_3)$ and~$E(\boldx)=(y_1,y_2,y_3)$, and then an erasure at entry~$1$ is equivalent to evaluating~$\tau'$ at the point~$(0,y_2,y_3)$. Similarly, an error in entry~$2$ is equivalent to evaluating~~$\tau'$ at~$(y_1,-y_2,y_3)$. 

As such, both errors and erasure can be seen as evaluation of the same coded neuron~$\tau'$ on a data point which is shifted along axis-parallel lines. Therefore, the encoded points must be far away from~$\cH(\boldv,\mu)$ in~$\ell_1$-distance. More precisely, since error and erasures do not cause any point to shift outside the closed hypercube~$[-1,1]^m$, it is only necessary to have large~$\ell_1$-distance from~$\cH'=\cH'(\boldv,\mu)=\cH(\boldv,\mu)\cap[-1,1]^m$. 

%\red{[Perhaps unnecessary?]} 
First, we provide the formula for the~$\ell_1$-distance of a point from a hyperplane.

\begin{lemma}\label{lemma:l1}\cite[Sec.~5]{LpDistance}
	For every~$\boldz\in\bR^m$ we have that $d_1(\boldz,\cH)=\frac{|\boldz\cdot\boldv^\top-\mu|}{\infnorm{\boldv}}$.
%	\begin{align*}
%	%d_1(\boldz,\cH)=\frac{\twonorm{\boldv}}{\infnorm{\boldv}}\cdot d_2(\boldz,\cH)=\frac{|\boldz\cdot\boldv^\top-\mu|}{\infnorm{\boldv}}.
%	d_1(\boldz,\cH)=\frac{|\boldz\cdot\boldv^\top-\mu|}{\infnorm{\boldv}}.
%	\end{align*}
\end{lemma}

Second, we provide a necessary and sufficient condition for the robustness of a coded neuron~$\tau'(E(\boldx))=\sign(E(\boldx)\boldv^\top-\mu)$. We denote the positive points of~$\tau$ by~$\cF^+$, and the negative points by~$\cF^-$.

\begin{theorem}\label{theorem:NecessaryAndSufficient}
	For a positive integer~$r$ and a neuron~$\tau(\boldx)=\sign(\boldx\boldw^\top-\theta)$, a coded neuron $\tau'(\boldx)=\sign(E(\boldx)\boldv^\top-\mu)$ is~$r$-robust if and only if
	\begin{align}\label{equation:NecessaryAndSufficientConditions}
	\sign(\boldx\boldw^\top-\theta)&=\sign(E(\boldx)\boldv^\top-\mu)\mbox{ for every }\boldx\in\bF_2^n,\nonumber\\
	r&\le d_1(E(\cF^+),\cH')\mbox{, and}\nonumber\\
	r&< d_1(E(\cF^-),\cH').
	\end{align}
\end{theorem}

\begin{proof}
	Assume that the conditions in~\eqref{equation:NecessaryAndSufficientConditions} hold. To show that~$\tau'$ is~$r$-robust we must show that~\eqref{equation:problem} holds for every~$\boldx\in\bF_2^n$ and every mutually disjoint~$\cS$ and~$\cT$ such that~$|\cT|+2|\cS|\le r$. Assuming for contradiction that~$\tau'$ is not~$r$-robust, there exists some~$\boldx\in\bF_2^n$ and corresponding sets~$\cS$ and~$\cT$ such that~$E(\boldx)$ is misclassified under erasures in~$\cT$ and errors in~$\cS$. 
	Since any set of errors or erasures keeps~$E(\boldx)$ inside~$[-1,1]^m$, it follows that this misclassification of~$E(\boldx)$ corresponds to moving it along an axis-parallel path~$P$ of length~$|P|=|\cT|+2|\cS|$, which crosses~$\cH'$. 
	
	If~$\boldx\in\cF^+$, then to attain misclassification we must have~$|P|>d_1(E(\boldx),\cH')$. However, this implies that~$r\ge |\cT|+2|\cS|=|P|>d_1(E(\boldx),\cH')\ge r$, a contradiction. If~$\boldx\in\cF^-$, then to attain misclassification we must have~$|P|\ge d_1(E(\boldx),\cH')$. However, this implies that~$r\ge |\cT|+2|\cS|=|P|\ge d_1(E(\boldx),\cH')>r$, another contradiction.
	
	Conversely, assume that~$\tau'$ is $r$-robust. Since~$r\ge0$, it follows that~$\tau'$ is in particular~$(0,0)$-robust, and hence according to~\eqref{equation:problem} if follows that $\sign(\boldx\boldw^\top-\theta)=\sign(E(\boldx)\boldv^\top-\mu)$ for every~$\boldx\in\bF_2^n$. Assume for contradiction that~$r>d_1(E(\cF^+),\cH')$, which implies that there exists~$\boldx\in\cF^+$ such that~$r>d_1(E(\boldx),\cH')$, and let $\cB_\boldx\triangleq \cB_1(E(\boldx),r)\cap [-1,1]^m$. This readily implies that some vertex~$\boldy$ of~$\cB_\boldx$ lies on the opposite side of~$\cH$. It can be proved (full proof will be given in future versions of this paper) that~$\boldy$ is an integer point, and that all such points correspond to erasures in some set~$\cT$ and errors in some set~$\cS$ such that~$|\cT|+2|\cS|\le r$. Therefore, the existence of~$\boldy$ contradicts the~$r$-robustness of~$\tau'$. The proof that~$r<d_1(E(\cF^-),\cH')$ is similar.%; however the vertex~$\boldy$ might reside \textit{on}~$\cH$. This case also leads to a contradiction, since~$\boldx\in\cF^-$ is classified negatively, whereas points on~$\cH$ are classified positively due to~$\sign(0)=1$.
\end{proof}

We conclude this section by showing that redundancy is \textit{necessary} for any nontrivial robustness. Since any non-constant neuron must have a positive point~$\boldx$ and a negative point~$\boldy$ such that~$d_H(\boldx,\boldy)=1$, and since any hyperplane must cross the convex hull of~$\boldx$ and~$\boldy$, the following is immediate.

\begin{lemma}\label{lemma:lowerBoundNoRedundancy}
	Unless $\tau(\boldx)=\sign(\boldx\boldw^\top-\theta)$ is constant, the solution~$(E,\boldv,\mu)=(\Id,\boldw,\theta)$ is $0$-robust.
\end{lemma}

Further, by denoting~$\delta=d_1(\bF_2^n,\cH(\boldw,\theta))$, we have that the relative distance of the solution~$(\Id,\boldw,\theta)$ is~$\delta/n$. However, computing~$\delta$ for a given neuron~$\tau$ is in general NP-complete, by a reduction to PARTITION~\cite{PARTITION}.

\section{A Few Elementary Solutions}\label{section:simple}
\subsection{Robustness by replication}\label{section:replication}
For a vector~$\boldv$, let~$\boldv_{(\ell)}$ be the result of concatenating~$\boldv$ with itself~$\ell$ times, and for~$E:\bF_2^n\to\bF_2^m$ let~$E_{(\ell)}:\bF_2^n\to\bF_2^{\ell m}$ be the function~$E_{(\ell)}(\boldx)=E(\boldx)_{(\ell)}$.

\begin{lemma}\label{lemma:ellreplication}
	Let~$(E,\boldv,\mu)$ be a solution with minimum distance~$d$. Then, for every positive integer~$\ell$, the solution~$(E_{(\ell)},\boldv_{(\ell)},\ell\mu)$ has distance~$\ell d$ and identical relative distance~$d/m$.
\end{lemma}
\begin{proof}
	According to Lemma~\ref{lemma:l1}, and since~$\infnorm{\boldv_{(\ell)}}=\infnorm{\boldv}$, we have that
\begin{align*}%\label{equation:replicationFormula}
d_1\left(E_{(\ell)}(\bF_2^n),\cH(\boldv_{(\ell)},\ell\mu)\right)&=
%\frac{\twonorm{\boldv_{(\ell)}}}{\infnorm{\boldv_{(\ell)}}}\cdot d_2\left(E_{(\ell)}(\bF_2^n),\cH(\boldv_{(\ell)},\ell\mu)\right)\nonumber\\
\frac{\min_{\boldx\in\bF_2^n}|E_{(\ell)}(\boldx)\boldv_{(\ell)}^\top-\ell\mu|}{\infnorm{\boldv}},
\end{align*}
%where the latter equality follows since~$\infnorm{\boldv_{(\ell)}}=\infnorm{\boldv}$, and by the formula for~$\ell_2$-distance of a point from a hyperplane. 
and since $E_{(\ell)}(\boldx)\boldv_{(\ell)}^\top=\ell\cdot E(\boldx)\boldv^\top$, it follows that this equals
\begin{align*}
\ell\cdot\frac{\min_{\boldx\in\bF_2^n}|E(\boldx)\boldv^\top-\mu|}{\infnorm{\boldv}}=\ell d,
\end{align*}
and thus the relative distance is $\ell d/\ell m=d/m$.
\end{proof}
Therefore, by applying the $\ell$-replication code~$E(\boldx)=\boldx_{(\ell)}$, one can obtain robustness but not increase the relative distance. Moreover, since computing the aforementioned~$\delta$ is NP-hard, explicit robustness guarantees are hard to come by. 

\subsection{Robustness from the Fourier spectrum}
Recall that every function~$f:\bF_2^n\to\bR$ (and in particular, every neuron~$\tau$) can be written as a linear combination~$f(\boldx)=\sum_{\cS\subseteq[n]}\hat{f}(\cS)\chi_\cS(\boldx)$, where~$\chi_\cS(\boldx)\triangleq \prod_{s\in\cS}x_s$ and~$\hat{f}(\cS)=\bE_\boldx \chi_\cS(\boldx)f(\boldx)$ for every~$\cS\subseteq[n]$. The vector~$\hat{\boldf}\triangleq (\hat{f}(\cS))_{\cS\subseteq[n]}$ is called the \textit{Fourier spectrum} of~$f$, and if~$f$ is Boolean then~$\twonorm{\hat{\boldf}}=1$. 
%The following statement applies to computation of \textit{any} Boolean~$f$, and is not restricted to~$\sign(\cdot)$ functions. To generalize our problem to all Boolean functions, we replace~$\sign(\boldx\boldw^\top-\theta)$ in~\eqref{equation:problem} by~$f(\boldx)$, and assume that~$f$ is given in some black box model. 
We denote by~$\hat{\boldf}_\varnothing$ the vector~$\hat{\boldf}$ with its~$\varnothing$-entry omitted, i.e., $\hat{\boldf}_\varnothing\triangleq (\hat{f}(\cS))_{\cS\subseteq[n],\cS\ne\varnothing}$. We refer to the following solution as the \textit{Fourier solution}.

\begin{lemma}
	For a neuron~$\tau$, the coded neuron~$\tau'(E(\boldx))\triangleq \sign(\sum_{\cS\subseteq[n]}\hat{\tau}(\cS)\chi_\cS(\boldx))$ has minimum distance~$\infnorm{\hat{\boldtau}_\varnothing}^{-1}$.% robustness~$\ceil{\infnorm{\hat{\boldf}_\varnothing}^{-1}}-1$.
\end{lemma}

\begin{proof}
	Notice that~$\tau'$ is defined by the encoding function~$E:\bF_2^n\to\bF_2^{2^n-1}$ such that~$E(\boldx)=(\chi_\cS(\boldx))_{\cS\subseteq[n],\cS\ne\varnothing}$, known as the \textit{punctured Hadamard} encoder. In addition, the respective halfspace is~$\cH=\cH(\hat{\boldtau}_\varnothing,-\hat{\tau}(\varnothing))\triangleq\{ \boldy\in\bF_2^{2^n-1}\vert \sum_{\cS\ne \varnothing}y_\cS\hat{\tau}(\cS)+\hat{\tau}(\varnothing)=0 \}$, where the coordinates of~$\bR^{2^n-1}$ are indexed by all nonempty subsets of~$[n]$. To find the minimum distance of the Fourier solution, we compute
	%	employ Lemma~\ref{lemma:sufficientCondition} once again. By Lemma~\ref{lemma:l1} we have
	\begin{align*}
	d_1(E(\bF_2^n),\cH)&
	%=\frac{\twonorm{\hat{\boldf}_\varnothing	}}{\infnorm{\hat{\boldf}_\varnothing}}\cdot d_2(E(\bF_2^n),\cH)
=\frac{\min_{\boldx\in\bF_2^n}|\hat{\boldtau}_\varnothing\cdot E(\boldx)+\hat{\tau}(\varnothing)|}{\infnorm{\hat{\boldtau}_\varnothing}}\\&=\frac{\min_{\boldx\in\bF_2^n}|\sum_{\cS\subseteq[n]}\hat{\tau}(\cS)\chi_\cS(\boldx)|}{\infnorm{\hat{\boldtau}_\varnothing}}=\infnorm{\hat{\boldtau}_\varnothing}^{-1},
	\end{align*}
	where the last equality follows since $\tau(\boldx)=\sum_{\cS\subseteq[n]}\hat{\tau}(\cS)\chi_\cS(\boldx)\in\{ \pm 1 \}$. %Therefore, $g$ is of robustness~$d$, where~$d$ is the largest integer which is smaller than~$\infnorm{\hat{\boldf}_\varnothing}^{-1}$, namely $\ceil{\infnorm{\hat{\boldf}_\varnothing}^{-1}}-1$.
\end{proof}

The relative distance of this solution is~$\infnorm{\hat{\boldtau}_\varnothing}^{-1}/(2^n-1)$, and notice that unlike replication (Subsection~\ref{section:replication}), it does not depend on the particular way in which~$\tau$ is given. However, this solution involves exponentially many redundant bits, and is therefore impractical. %The Fourier expansion can be approximated by using the formula~$\hat{f}(\cS)=\bE_{\boldx\in\bF_2^n}f(\boldx)\chi_\cS(\boldx)$ and applying concentration bounds, where the value of~$f$ is calculated from its black-box representation. However, this requires exponentially many calculations, and exponential redundancy, and hence is applicable only for small values of~$n$. 

\section{Robustness from the Parity Code}\label{section:parity}
The discussion in this section applies to DNNs that employ \textit{binary neurons}, i.e., neurons in which~$\boldw\in \bF_2^n$. This family of DNNs includes, as a strict subset, the recently popularized \textit{binarized neural networks}~\cite{BNN}. Later on, we generalize the solution to all neurons, and show its superiority over replication in cases where $\onenorm{\boldw}$ is bounded. Specifically, we show that the parity code attains relative distance~$2/(n+1)$, which outperforms replication.

We denote
\begin{align*}
\cU\triangleq\{ \tau:\bF_2^n\to\bF_2\vert \tau(\boldx)=\sign(\boldx\boldw^\top-\theta)\mbox{ and }\boldw\in\bF_2^n \},
\end{align*}
and since~$\boldx\boldw^\top\in\{-n,-n+2,\ldots,n\}$ for every~$\boldx$ and~$\boldw$ in~$\bF_2^n$, it follows that for every~$\tau\in\cU$ one can round the respective~$\theta$ to the nearest value\footnote{More precisely, if $\theta\in(-n+2t,-n+2t+2]$ for an integer~$0\le t\le n$, then~$\theta$ is replaced by~$-n+2t+1$. If $\theta\le -n$ it is replaced by~$-n-1$, and if $\theta>n $ it is replaced by~$n+1$.} in~$\{-n-1,-n+1,\ldots,n+1\}$ without altering~$\tau$.
Hence, we assume without loss of generality that all given $\theta$'s are in~$\{-n-1,-n+1,\ldots,n+1\}$. With this choice of~$\theta$, any function~$f\in\cU$ has~$\delta=1$, since
\begin{align*}%\label{equation:delta=1}
\delta&=d_1(\bF_2^n,\cH(\boldw,\theta))=\frac{\min_{\boldx\in\bF_2^n}|\boldx\boldw^\top-\theta|}{\infnorm{\boldw}}\\
&=\min_{\boldx\in\bF_2^n}|\boldx\boldw^\top-\theta|=1
\end{align*}
by Lemma~\ref{lemma:l1}. Thus, replication achieves relative distance~$1/n$ (e.g., $2$-replication achieves~$1$-robustness with~$m=2n$).

We also employ the following notations from Boolean algebra. For~$\boldx,\boldw\in\bR^n$ let~$\boldx\oplus\boldw$ denote their pointwise product, which amounts to Boolean sum if both~$\boldx$ and~$\boldw$ are in~$\bF_2^n$. Further, for~$\boldx\in\bF_2^n$ we let~$w_H(\boldx)$ be the number of~$(-1)$-entries in~$\boldx$, known as \textit{Hamming weight}. %, and let~$\1$ be the all~$1$'s vector, whose length is understood from the context. 
The next two lemmas, whose proofs will appear in future version of this paper, demonstrate that functions in~$\cU$ depend only on~$w_H(\boldx\oplus\boldw)$.

\begin{lemma}\label{lemma:Hammingweight}
	For every~$\boldx$ and~$\boldw$ in~$\bF_2^n$ we have~$\boldx\boldw^\top=n-2w_H(\boldx\oplus\boldw)$.
\end{lemma}
%\begin{proof}
%	Notice that~$\1\boldx^\top=n-2w_H(\boldx)$, and substitute~$\boldx\oplus \boldw$ for~$\boldx$.
%\end{proof}

\begin{lemma}\label{lemma:fthresholds}
	For every~$\tau\in\cU$ and every~$\boldx\in\bF_2^n$ we have
	\begin{align*}
	\tau(\boldx)=
	\begin{cases}
	\phantom{-}1 & w_H(\boldx\oplus\boldw)\le\frac{n-\theta-1}{2}\\
	-1 & w_H(\boldx\oplus\boldw)\ge\frac{n-\theta+1}{2}
	\end{cases}.
	\end{align*}
\end{lemma}
%\begin{proof}
%	Notice that~$\boldx\boldw^\top-\theta$ is always an integer, and hence~$f(\boldx)=-1$ if and only if
%	\begin{align*}
%	\boldx\boldw^\top-\theta&\le -1\\
%	n-2w_H(\boldx\oplus\boldw)-\theta&\le -1\mbox{ (by Lemma~\ref{lemma:Hammingweight})}\\
%	w_H(\boldx\oplus\boldw)&\ge\tfrac{n-\theta+1}{2}.
%	\end{align*}
%	On the other hand, $f(\boldx)=1$ if and only if
%	\begin{align}\label{equation:fx=1}
%	\boldx\boldw^\top-\theta&\ge 0\nonumber\\
%	n-2w_H(\boldx\oplus\boldw)-\theta&\ge 0\mbox{ (by Lemma~\ref{lemma:Hammingweight})}\nonumber\\
%	w_H(\boldx\oplus\boldw)&\le \tfrac{n-\theta}{2},
%	\end{align}
%	and since~$n-\theta$ is odd and~$w_H(\boldx\oplus\boldw)$ is an integer, we can safely deduct~$1/2$ from the right hand side of~\eqref{equation:fx=1} and conclude this case as well. 		
%\end{proof}

In this section we let~$m=n+1$ and define~$E:\bF_2^n\to\bF_2^{n+1}$ as the \textit{parity} encoder $E(\boldx)=(x_1,\ldots,x_n,\chi_{[n]}(\boldx))$. Then, we let~$\theta'\triangleq \frac{n-\theta-1}{2}$, and define the \textit{parity} solution~$(E,\boldv,\mu)$, where
\begin{align*}
\boldv&=(w_1,\ldots,w_n,(-1)^{\theta'}\chi_{[n]}(\boldw)),\mbox{ and}\\
\mu&=\theta.
\end{align*}

\begin{lemma}
	The relative distance of the parity solution is~$2/(n+1)$.
\end{lemma}

\begin{proof}
	We show that~$d_1(E(\bF_2^n),\cH)=2$, where~$\cH=\cH(\boldv,\theta)$. Since~$\infnorm{\boldv}=1$, Lemma~\ref{lemma:l1} implies that
	\begin{align*}
	d_1(E(\bF_2^n),\cH)
	%=\frac{\twonorm{\boldv}}{\infnorm{\boldv}}\cdot d_2(E(\bF_2^n),\cH)=\twonorm{\boldv}\cdot \frac{\min_{\boldx\in\bF_2^n}|E(\boldx)\cdot\boldv^\top-\theta|}{\twonorm{\boldv}}\\
	=\min_{\boldx\in\bF_2^n}|E(\boldx)\cdot\boldv^\top-\theta|,
	\end{align*}
	and hence it suffices to show that~$|E(\boldx)\boldv^\top-\theta|\ge 2$ for every~$\boldx\in\bF_2^n$ and that the coded neuron always correctly computes~$\tau$. Since $\chi_{[n]}(\boldw)\cdot\chi_{[n]}(\boldx)=\chi_{[n]}(\boldw\oplus\boldx)=(-1)^{w_H(\boldw\oplus\boldx)}$, Lemma~\ref{lemma:Hammingweight} implies that
	\begin{align}\label{equation:parityProof}
	E(\boldx)\boldv^\top-\theta&=\boldx\boldw^\top+(-1)^{\theta'}\chi_{[n]}(\boldw)\cdot\chi_{[n]}(\boldx)-\theta\\\nonumber
	&=n-2w_H(\boldx\oplus\boldw)+(-1)^{\theta'+w_H(\boldx\oplus\boldw)}-\theta,
	\end{align}
	and we distinguish between the next four cases.
	\begin{itemize}
		\item[] \textbf{Case 1:} If $w_H(\boldx\oplus\boldw)\le \frac{n-\theta-1}{2}-1$, then
		\begin{align*}
		\eqref{equation:parityProof}&\ge n-(n-\theta-1)+2+(-1)^{\theta'+w_H(\boldw\oplus\boldx)}-\theta\\
		&=3+(-1)^{\theta'+w_H(\boldw\oplus\boldx)}\ge 2.
		\end{align*} 
		\item[] \textbf{Case 2:} If $w_H(\boldx\oplus\boldw)= \frac{n-\theta-1}{2}$, then
		\begin{align*}
		\eqref{equation:parityProof}&=n-(n-\theta-1)+(-1)^{\theta'-\frac{n-\theta-1}{2}}-\theta\\&=1+(-1)^0=2.
		\end{align*}
		\item[] \textbf{Case 3:} If $w_H(\boldx\oplus\boldw)= \frac{n-\theta+1}{2}$, then
		\begin{align*}
		\eqref{equation:parityProof}&=n-(n-\theta+1)+(-1)^{\theta'+\frac{n-\theta+1}{2}}-\theta\\
		&=-1+(-1)^{n-\theta}= -2,
		\end{align*}
		where the last equality follows since~$n-\theta$ is always odd.
		\item[] \textbf{Case 4:} If $w_H(\boldx\oplus\boldw)\ge \frac{n-\theta+1}{2}+1$, then
		\begin{align*}
		\eqref{equation:parityProof}&\le n-(n-\theta+1)-2+(-1)^{\theta'+w_H(\boldx\oplus\boldw)}-\theta\\
		&=-3+(-1)^{\theta'+w_H(\boldx\oplus\boldw)}\le -2.
		\end{align*}
	\end{itemize}
	Now, it follows from Lemma~\ref{lemma:fthresholds} and from the first two cases that $\sign(E(\boldx)\boldv^\top-\theta)=1$ whenever~$\tau(\boldx)=1$. Similarly, the latter two cases imply that $\sign(E(\boldx)\boldv^\top-\theta)=-1$ whenever~$\tau(\boldx)=-1$. Therefore, the coded neuron~$\tau'(E(\boldx))=\sign(E(\boldx)\boldv^\top-\theta)$ correctly computes~$\tau$ on all inputs with minimum distance~$d=2$, and the claim follows.
\end{proof}

By using the parity solution, one can attain~$1$-robustness, i.e., robustness against \textit{any single (adversarial) erasure}, by adding only one bit of redundancy. In contrast, to attain~$1$-robustness by using replication (Subsection~\ref{section:replication}), one should add~$n$ bits of redundancy. Moreover, it is readily verified that due to Lemma~\ref{lemma:lowerBoundNoRedundancy}, the suggested solution is optimal in terms of the length~$m=n+1$ among all~$1$-robust solutions.

Since the parity function~$E$ is universal to all binary neurons, every DNN which comprises of binary neurons (and in particular, binarized DNNs) can be made robust to adversarial tampering in its input. Furthermore, to employ this technique for error \textit{inside} the DNN, one should add a single parity gate in every layer (see Fig.~\ref{figure:InnerRedundancy}). If one wishes to construct DNNs by \textit{only} using neurons, a classic result by Muroga~\cite{Muroga} shows how to implement the parity function by using neurons.

\subsection{Generalized parity for all neurons}

The following generalizes the parity code, and requires \textit{integer weights}. Since every neurons has a representation with only integer weights~\cite[Exercise.~5.1]{AnalysisOfBoolean}, it applies to all neurons. However, superiority to replication in terms of relative distance is guaranteed only if $\onenorm{\boldw}<\frac{2n}{\delta}-1$. The proof will appear in future versions of this paper.
\begin{theorem}
	The relative distance of the solution~$(E_\boldw,\boldv,\theta)$ is~$2/(\onenorm{\boldw}+1)$, where
	\begin{align*}
		\boldv&=(\1_\boldw,(-1)^{\theta'}\chi_{[\onenorm{\boldw}]}(\1_\boldw)),\nonumber\\
		\1_{\boldw}&=( \underbrace{\sign(w_1),\ldots,\sign(w_1)}_{|w_1|\mbox{ times}},\ldots,  \underbrace{\sign(w_n),\ldots,\sign(w_n)}_{|w_n|\mbox{ times}} ),%\in\bF_2^{\nu},
		\end{align*}
		and 
		\begin{align*}
		E_{\boldw}(\boldx) = \left( \underbrace{x_1,\ldots,x_1}_{|w_1|\mbox{ times}},\ldots,  \underbrace{x_n,\ldots,x_n}_{|w_n|\mbox{ times}},\prod_{i=1}^n x_i^{w_i\bmod 2} \right).%\in\bF_2^{\nu}.
		\end{align*}
\end{theorem}

\section{Discussion and Future Research}
In this paper we studied a novel approach for combating noise in DNNs with error correcting codes, established basic framework, and presented several solutions. This work can be seen as an extension of the recently popularized \textit{coded computation} topic, which concerns computation over coded data in distributed environments, into the realm of neural computation. A plethora of questions remain widely open:
%It is worth pointing out that the use of the parity encoder for combating noise in DNNs outperforms its use in digital communication in some sense--in the latter one can only \textit{detect} but not \textit{correct} noise, whereas in the former the parity code can guarantee successful ``decoding'', i.e., successful computation under noise. 
\begin{enumerate}
	\item Extend the above framework to other activation functions, and to sigmoid functions in particular.
	\item Develop solutions with relative distance greater than~$2/(n+1)$ for binarized neurons.
	\item Find other families of neurons for which robustness can be guaranteed.
	\item Extend Lemma~\ref{lemma:lowerBoundNoRedundancy} to other parameter regimes, i.e., establish fundamental trade-offs between the parameters~$n,m$, and~$d$.
\end{enumerate}

\end{document}